\documentclass{article}
\usepackage{CISE,times}

\usepackage{amsmath,amsfonts,bm}

\def\eqref#1{equation~\ref{#1}}

\def\1{\bm{1}}

\DeclareMathAlphabet{\mathsfit}{\encodingdefault}{\sfdefault}{m}{sl}
\SetMathAlphabet{\mathsfit}{bold}{\encodingdefault}{\sfdefault}{bx}{n}

\usepackage[utf8]{inputenc}
\usepackage[T1]{fontenc}

\usepackage{amsmath,amssymb,amsthm}

\usepackage{algorithm}
\usepackage{algorithmic}

\usepackage{graphicx}
\usepackage{booktabs}
\usepackage{array}
\usepackage{tabularx}

\usepackage{microtype}
\usepackage{xcolor}
\usepackage{url}
\usepackage{hyperref}

\newtheorem{theorem}{Theorem}

\title{Reliable Self-Evolution with \\ Imperfect Proxy Rewards.}

\author{
\noindent Kangjun Noh$^1$, Soyu Kim$^1$, Kyungwoo Song$^1$$^{\dagger}$ \\
$^1$Yonsei University \\
}

\iclrfinalcopy

\hypersetup{
    pdftitle={Reliable Self-Evolution with Imperfect Proxy Rewards.},
    pdfauthor={Kangjun Noh, Soyu Kim, Kyungwoo Song}
}

\begin{document}

\maketitle

\begingroup
\renewcommand{\thefootnote}{\fnsymbol{footnote}}
\footnotetext[2]{Corresponding author.}
\endgroup

\begin{abstract}
Large language model (LLM)-based self-evolving search is a promising approach to scientific discovery. However, high-fidelity evaluation of every candidate is prohibitively expensive in some domains. Self-evolving systems in such settings therefore rely on low-cost but imperfect proxy rewards, which may assign high scores to infeasible candidates. These false positives may contaminate both the final output and the feedback used to guide subsequent generations. This motivates statistically calibrated reward intervals for more reliable self-evolving search. We propose Conformal Interval-Driven Self-Evolution (CISE), which constructs candidate-specific reward intervals using conditional conformal inference and iteration-wise online density-ratio estimation. CISE uses conservative interval-based rewards for evolutionary feedback and returns candidates only when all required property intervals lie entirely within their respective feasible regions. We derive fixed-iteration coverage results under explicit assumptions of independence and covariate shift. We evaluate CISE on three self-evolving search tasks in materials science. In our experiments, all candidates returned by CISE are true positives under high-fidelity evaluation, whereas the baselines return more candidates but include false positives. These results highlight the value of a smaller, more precise shortlist when downstream validation budgets are limited. Our repository is available at \url{https://github.com/MLAI-Yonsei/CISE.git}.
\end{abstract}

\section{Introduction}
\label{sec:introduction}
Self-evolving search uses large language models (LLMs) to generate candidates and improve subsequent candidates using feedback from previous evaluations. This paradigm has recently emerged as an approach to scientific discovery~\citep{romera2024mathematical,novikov2025alphaevolvecodingagentscientific,liu2026evoxmetaevolutionautomateddiscovery}.
These methods retain previously evaluated candidates together with their rewards and reuse them to guide subsequent search.
The effectiveness of self-evolution therefore depends on the reliability of the rewards that determine which candidates are treated as successful and reused in later iterations.

However, obtaining the true reward for every generated candidate is often slow or expensive in iterative search.
In such settings, self-evolving systems often rely on computationally inexpensive proxy rewards that approximate the true reward.
For example, in materials discovery, evaluating every candidate with high-fidelity density functional theory (DFT) calculations or experimental validation is prohibitively expensive. Property prediction models are therefore often used as proxy rewards during search~\citep{abhyankar2026llemaevolutionarysearchllms}.
Although proxy rewards make iterative search practical, prediction errors may produce false positives: candidates that satisfy the target constraints under the proxy reward but fail under the true reward.
In self-evolving search, such errors affect more than the returned candidates.
Once a false positive is assigned a favorable reward and stored as a successful example, it is reused to guide subsequent generations.
This false-positive propagation may steer the search toward regions where the proxy reward is overly optimistic rather than toward candidates with genuinely high true rewards.

This motivates uncertainty-aware reward estimates rather than relying solely on point-valued proxy rewards.
Conformal inference~\citep{vovk2005algorithmic,angelopoulos2022gentleintroductionconformalprediction} provides a principled way to construct such estimates by converting model predictions into calibrated intervals with statistical coverage guarantees.
In particular, conditional conformal inference~\citep{gibbs2024conformalpredictionconditionalguarantees} incorporates candidate characteristics into calibration, which is useful when proxy-reward reliability varies across candidates. These calibrated intervals support conservative decisions about which candidates guide subsequent evolution and which candidates are returned as outputs.

We propose Conformal Interval-Driven Self-Evolution (CISE), a self-evolving framework that replaces point-valued proxy rewards with candidate-specific calibrated reward intervals.
CISE uses conditional conformal inference to construct candidate-specific reward intervals with adaptive widths.
The resulting reward interval is stored in the archive and used to guide subsequent generations, while only candidates whose entire reward intervals satisfy the task-specific constraints are returned as outputs.
Because the candidate distribution can change as the archive, selected parents, and generation context evolve, CISE additionally performs iteration-wise online density-ratio estimation~\citep{JMLR:v10:kanamori09a} to characterize the difference between the source distribution and the current candidate distribution and incorporates the estimated ratio into conditional conformal inference. As a result, calibrated reward intervals are used for both evolutionary feedback and final candidate selection.

Theoretically, we analyze fixed-iteration marginal coverage under explicit independence and covariate-shift assumptions and characterize how density-ratio estimation error affects the resulting guarantee. Empirically, we evaluate CISE on three self-evolving material-property search tasks and validate the returned candidates using high-fidelity DFT calculations. Across the reported experiments, CISE trades output yield for higher observed DFT-validated output precision: all CISE outputs with resolved DFT evaluations are true positives, whereas the baseline methods return false positives. We further show under a matched downstream validation budget that the reliability improvement cannot be explained solely by CISE returning fewer candidates. These results demonstrate the potential of calibrated reward intervals to make self-evolving search more reliable when accurate true rewards are too expensive to use throughout evolution. To the best of our knowledge, CISE is the first framework to use conformal proxy-reward intervals for both evolutionary feedback and final candidate selection in LLM-based self-evolving candidate search.

\section{Related Work}
\label{sec:related_work}
Self-evolving search reuses previously generated candidates and their rewards to guide subsequent generations across domains, including program discovery~\citep{romera2024mathematical,novikov2025alphaevolvecodingagentscientific}, video-prompt optimization~\citep{ji2024promptavideopromptvideodiffusion}, scientific hypothesis generation~\citep{Gottweis_2026}, and materials discovery~\citep{abhyankar2026llemaevolutionarysearchllms}. While some settings provide directly verifiable rewards, others rely on learned or LLM-generated proxy rewards to enable efficient iterative search. In such cases, proxy-favored candidates may not correspond to true improvements, and reusing them as successful examples can propagate reward errors across subsequent generations. This motivates accounting for proxy-reward uncertainty in both evolutionary feedback and final selection.

Recent work has sought to improve the reliability of proxy evaluations and rewards through calibration and uncertainty-aware optimization. For LLM-as-a-judge, post-hoc score calibration learns regression models from limited human feedback to adjust judge scores using their original scores and rationales~\citep{sahoo2026quantitative}, while linear probes on judge hidden states provide calibrated confidence estimates for the correctness of their judgments~\citep{radharapu2025calibratingllmjudgeslinear}. Related work on learned reward models addresses uncertainty and overoptimization directly: ensemble-based conservative objectives use disagreement among reward models to reduce optimization against errors in an imperfect reward model~\citep{coste2024rewardmodelensembleshelp}. Process reward models have also been calibrated using quantile regression to better align their predictions with empirical success probabilities and support uncertainty-aware allocation of inference-time computation~\citep{park2025knowdontknowuncertainty}. These approaches improve the calibration, interpretation, or robustness of model-derived rewards. However, they do not statistically control false-positive rewards under task-specific true-reward constraints in iterative self-evolving search.

Conformal inference provides a complementary framework for uncertainty quantification by converting model predictions into calibrated prediction sets or intervals with finite-sample coverage guarantees under explicit assumptions~\citep{vovk2005algorithmic,angelopoulos2022gentleintroductionconformalprediction}. Conditional conformal inference formulates conditional validity through classes of covariate shifts and provides finite-sample guarantees for shifts represented by a finite-dimensional calibration function space~\citep{gibbs2024conformalpredictionconditionalguarantees}. More recently, conformal inference has been applied to LLM-as-a-judge evaluation, replacing point-valued judge scores with calibrated intervals that explicitly represent evaluation uncertainty~\citep{sheng-etal-2025-analyzing,noh2026reliable}. Building on this line of work, we use conformal inference to construct calibrated intervals for proxy rewards and use these reward intervals in both evolutionary feedback and final candidate selection in self-evolving search.

\section{Preliminaries: Self-Evolving Agent for Material Search with Proxy Reward}
\label{sec:self_evolving_framework}
We consider a self-evolving agent that searches a candidate space $\mathcal X$ for solutions satisfying task-specific constraints. For each candidate $x\in\mathcal X$, $r(x)\in\mathbb R$ denotes a scalar true reward, and $\hat r(x)\in\mathbb R$ denotes an inexpensive and imperfect proxy reward. During search, generated candidates are evaluated using the proxy reward. The task-specific constraints are specified before evolution. Let $\mathcal F_r$ and $\mathcal F_{\hat r}$ denote the sets of candidates satisfying these constraints according to the true and proxy rewards, respectively. Membership in $\mathcal F_{\hat r}$ is determined by computing $\hat r(x)$ and checking it against the specified constraints. A candidate $x\in\mathcal F_{\hat r}\setminus\mathcal F_r$ is a false positive: it satisfies the constraints under the proxy reward but violates them under the true reward.

The archive $\mathcal H_t$ stores candidates generated through iteration $t$ together with their proxy rewards. The parent-selection policy $\pi_t(\cdot\mid\mathcal H_t)$ selects $K$ candidates from the archive to form the parent set $\mathcal S_t$, which guides candidate generation at the next iteration. We treat $\{\pi_t\}_{t=0}^{T}$ as a given sequence of selection rules; we do not consider a separate policy-learning step. As the archive and its stored rewards change, the selected parents and the resulting candidate generation can nevertheless change across iterations. Algorithm~\ref{alg:proxy_self_evolution} summarizes the proxy-reward-driven search process. At iteration $t$, the agent uses the previous parents $\mathcal S_{t-1}$ and their stored rewards to generate a batch $\mathcal C_t=\{x_{t,1},\ldots,x_{t,M}\}$ of $M$ new candidates. The proxy reward is computed for each candidate, and the resulting candidate-reward pairs are added to the archive. Candidates whose proxy rewards satisfy the task-specific constraints are also added to the cumulative output set, giving $\mathcal O_t=\mathcal O_{t-1}\cup(\mathcal C_t\cap\mathcal F_{\hat r})$. The updated archive is then used to select the next parent set. The archive and output set serve different purposes: the archive retains generated candidates as search experience, whereas the output set contains only candidates that satisfy the proxy-reward constraints.

\begin{figure*}[t]
    \centering
    \includegraphics[width=\textwidth]{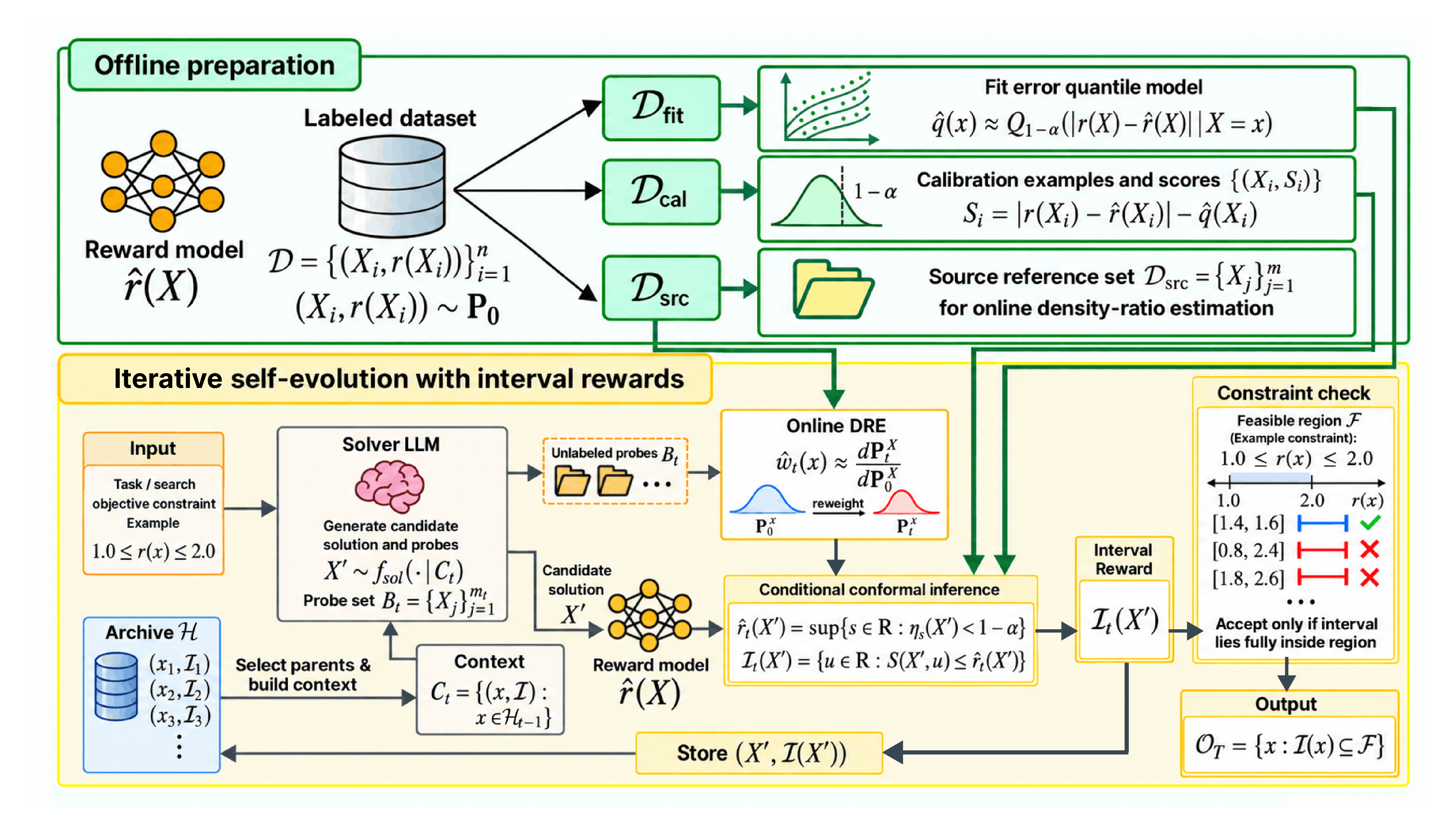}
    \caption{Overview of CISE. Offline, we fit the conditional proxy-error quantile using $\mathcal D_{\mathrm{fit}}$, $\mathcal D_{\mathrm{cal}}$ provides calibration scores, and $\mathcal D_{\mathrm{src}}$ provides the source reference for density-ratio estimation. During evolution, unlabeled probes are used to estimate the current distribution shift, which is incorporated into conditional conformal inference to construct candidate-specific reward intervals. These intervals guide subsequent evolution and candidates are returned only when the entire interval satisfies the task constraint.}
    \label{fig:cise_method}
\end{figure*}

\section{Methodology}
\subsection{Problem Setup}
\label{sec:problem_setup}
We are given a labeled source dataset $\mathcal D=\{(X_i,r(X_i))\}_{i=1}^{n}$ of independent candidate-reward pairs drawn from a joint distribution $\mathbf P_0$, with candidate marginal $\mathbf P_0^X$. We partition $\mathcal D$ into three mutually disjoint subsets: $\mathcal D_{\mathrm{fit}}$, $\mathcal D_{\mathrm{cal}}$, and $\mathcal D_{\mathrm{src}}$. The fitting set $\mathcal D_{\mathrm{fit}}$ is used to train a model $\hat q(x)$ that predicts an upper conditional quantile of the absolute proxy error $|r(x)-\hat r(x)|$. The calibration set $\mathcal D_{\mathrm{cal}}$ is used to construct conformal reward intervals, while $\mathcal D_{\mathrm{src}}$ provides the source reference for density-ratio estimation. Only the candidate information in $\mathcal D_{\mathrm{src}}$ is used for density-ratio estimation. The proxy reward model, fitted quantile model, and candidate representation are fixed before evolution.

For a fixed iteration $t$, let $\mathbf P_t$ denote the joint distribution of a candidate and its true reward under the current archive, selected parents, and generation configuration. For a fresh candidate--reward pair $(X',r(X'))\sim\mathbf P_t$, CISE uses conditional conformal inference~\citep{gibbs2024conformalpredictionconditionalguarantees} to construct an iteration-specific prediction interval $\mathcal I_t(X')$ for the true reward. This interval represents uncertainty in the proxy reward and is used for both evolutionary feedback and output selection.

For evolutionary feedback, a candidate with a nonempty, bounded interval is evaluated conservatively over that interval using a task-specific rule. The resulting interval reward is stored with the candidate in $\mathcal H_t$. The parent-selection policy $\pi_t(\cdot\mid\mathcal H_t)$ uses these stored rewards to select the next parents, and subsequent generation uses the selected parents and their rewards as feedback. For output selection, a candidate is added to $\mathcal O_t$ only when its interval is nonempty and bounded and every reward value in the interval satisfies the task-specific constraints. A lower-bound constraint is checked using the lower endpoint, an upper-bound constraint using the upper endpoint, and a bounded-range constraint requires the entire interval to lie within the prescribed range. Thus, a candidate with a valid interval may be retained in the archive even when it does not satisfy the output criterion.

Our target is iteration-wise marginal coverage: for a fresh candidate--reward pair at a fixed iteration $t$, the true reward should lie in its prediction interval with probability at least $1-\alpha$,
\begin{equation}
    \label{eq:coverage_iterative}
    \mathbb P\!\left(r(X')\in\mathcal I_t(X')\right)
    \ge 1-\alpha,
    \qquad t=1,\ldots,T.
\end{equation}
The probability is marginal over the source data and the fresh candidate--reward pair under the assumptions stated below. The guarantee is evaluated separately at each iteration and does not assert simultaneous coverage over the full search trajectory. We first analyze this target under exchangeability and then characterize the effect of covariate shift and density-ratio estimation error.

Let $A_t$ denote acceptance at iteration $t$. Since acceptance requires the entire interval to satisfy the task constraint,
\[
A_t\cap\{X'\notin\mathcal F_r\}
\subseteq
\{r(X')\notin\mathcal I_t(X')\},
\]
and hence $\mathbb P(A_t,X'\notin\mathcal F_r)\le\alpha$ whenever \eqref{eq:coverage_iterative} holds. This controls the joint acceptance-and-false-positive event, not the false-positive probability conditional on acceptance.

\subsection{Conformal Interval-driven Self-Evolution under Exchangeability}
\label{sec:exchangeable_cci}

We first consider an exchangeable reference setting with $\mathbf P_t=\mathbf P_0$ for a fixed iteration $t$ before considering covariate shift. Write the calibration pairs as $\{(X_i,r(X_i))\}_{i=1}^{n_{\mathrm{cal}}}$, and let $(X',r(X'))\sim\mathbf P_t$ denote an independent fresh candidate-reward pair. The equality $\mathbf P_t=\mathbf P_0$ means that the calibration candidates and the newly generated candidate follow the same joint distribution over material structures and true rewards. We hold the archive, selected parents, and generation configuration fixed, and assume that the resulting candidate distribution is independent of $\mathcal D_{\mathrm{cal}}$.

Let $S(x,u;\hat r)$ denote a fixed nonconformity score measuring the discrepancy between a hypothetical reward value $u$ and the proxy prediction $\hat r(x)$. Let $\hat q(x)$ denote a fitted conditional quantile of the absolute proxy error. We define the score as $S(x,u;\hat r)=|u-\hat r(x)|-\hat q(x)$, so that larger values indicate discrepancies that are unusually large relative to the prediction error expected at $x$. For each calibration pair $(X_i,r(X_i))$, we write $S_i=S(X_i,r(X_i);\hat r)$. Let $\mathcal G$ be a finite-dimensional linear function space spanned by the constant function and fixed functions of the candidate representation $\phi(x)$. For example, $\mathcal G$ may be chosen as $\mathcal G=\operatorname{span}\{1,\phi_1(x),\ldots,\phi_d(x)\}$. This space allows the calibration adjustment to vary with candidate characteristics.

Following \citet{gibbs2024conformalpredictionconditionalguarantees}, we calibrate the scores through an augmented quantile regression that includes the fresh candidate. For a fixed $X'$ and an imputed test score $s\in\mathbb R$, define
\begin{equation}
\widehat q_s
\in
\underset{q\in\mathcal G}{\arg\min}
\left\{
\sum_{i=1}^{n_{\mathrm{cal}}}
\ell_\alpha\!\left(q(X_i),S_i\right)
+
\ell_\alpha\!\left(q(X'),s\right)
\right\},
\label{eq:augmented_quantile_calibration}
\end{equation}
where
$\ell_\alpha(a,b)=(1-\alpha)(b-a)_++\alpha(a-b)_+$
is the pinball loss for the $(1-\alpha)$-quantile, with
$(v)_+=\max\{v,0\}$.
The additional term incorporates the fresh candidate into calibration using a hypothetical score, allowing the adjustment to account for $X'$ before its true reward is observed.

Let $\eta_s(X')\in[-\alpha,1-\alpha]$ denote the dual variable associated with the fresh candidate in~\eqref{eq:augmented_quantile_calibration}. This variable is nondecreasing in $s$. The nonrandomized score cutoff and the resulting reward interval are
\begin{equation}
\begin{aligned}
\widehat\tau_t(X')
&=
\sup\left\{
s\in\mathbb R:
\eta_s(X')<1-\alpha
\right\},
\\
\mathcal I_t(X')
&=
\left\{
u\in\mathbb R:
S(X',u)\le\widehat\tau_t(X')
\right\}.
\end{aligned}
\label{eq:cci_interval}
\end{equation}
The cutoff is computed from the dual boundary of the unpenalized augmented problem. Inverting the chosen absolute-error score yields an interval centered at $\hat r(X')$, with its width determined by the fixed normalization, fitted base quantile, and calibrated cutoff. An empty interval leads to abstention under the output rule in Section~\ref{sec:problem_setup}.

\begin{theorem}[Iteration-wise marginal coverage under exchangeability]
\label{thm:exchangeable_cci}
For a fixed iteration $t$, assume that the calibration pairs are i.i.d.\ draws from $\mathbf P_0$, and that the fresh pair $(X',r(X'))$ is an independent draw from $\mathbf P_t=\mathbf P_0$. Suppose that the score $S$ and finite-dimensional linear space $\mathcal G$ are fixed independently of these pairs, and that $\mathcal G$ contains the constant function. If~\eqref{eq:augmented_quantile_calibration} is solved exactly with a dual-solution rule symmetric under permutations of the augmented observations, then the interval in~\eqref{eq:cci_interval} satisfies
\[
\mathbb P
\left(
r(X')\in\mathcal I_t(X')
\right)
\ge 1-\alpha.
\]
\end{theorem}

At the true test score, the augmented observations are exchangeable. The constant function in $\mathcal G$ imposes a zero-sum constraint on the optimal dual variables. Together with their bounds, this limits the fraction reaching the upper boundary $1-\alpha$ to at most $\alpha$. The result establishes~\eqref{eq:coverage_iterative} in the reference setting. The next subsection adapts the calibration function space to the current candidate distribution when $\mathbf P_t\neq\mathbf P_0$.

\subsection{Conformal Interval-driven Self-Evolution under Evolutionary Covariate Shift}
\label{sec:shift_aware_cci}

As evolution proceeds, changes in the archive, selected parents, and generation strategy alter the distribution of materials proposed by the agent. For a fixed iteration $t$, let $\mathbf P_t$ denote the joint distribution of a candidate and its true reward under the current archive, parent set, and generation configuration. We consider evolutionary covariate shift satisfying
\begin{equation}
\mathbf P_t(r(X)\mid X)
=
\mathbf P_0(r(X)\mid X),
\qquad
\mathbf P_t^X\ll\mathbf P_0^X,
\qquad
w_t(x)
=
\frac{d\mathbf P_t^X}{d\mathbf P_0^X}(x)
\le W,
\label{eq:evolutionary_covariate_shift}
\end{equation}
where $W<\infty$ is a bounded-overlap constant. Thus, evolution may change how frequently different material structures are generated while preserving the conditional relationship between a structure and its true reward. The density ratio $w_t$ describes this change relative to the source distribution.

At iteration $t$, we generate an unlabeled probe batch $\mathcal B_t$ from $\mathbf P_t^X$ using the current parents and their stored rewards. The probe batch is used only for density-ratio estimation, and fresh candidates are subsequently generated for interval evaluation and archive updates. The archive, parents, and generation configuration remain fixed throughout both stages. We assume that the probe candidates and the subsequent candidates are independent draws from the same current distribution, with $\mathcal D_{\mathrm{cal}}$ independent of both stages.

Let $\phi(X)$ be the fixed candidate representation, and let
$\mathcal W=\{w_\theta:w_\theta(x)=h_\theta(\phi(x))\in[0,W],\theta\in\Theta\}$
be a bounded density-ratio model class. We estimate the current ratio through least-squares importance fitting~\citep{JMLR:v10:kanamori09a}:
\begin{equation}
\widehat\theta_t
\in
\underset{\theta\in\Theta}{\arg\min}
\left\{
\frac{1}{2|\mathcal D_{\mathrm{src}}|}
\sum_{(X_i,r(X_i))\in\mathcal D_{\mathrm{src}}}
w_\theta(X_i)^2
-
\frac{1}{|\mathcal B_t|}
\sum_{X'\in\mathcal B_t}
w_\theta(X')
+
\lambda\Omega(\theta)
\right\},
\qquad
\widehat w_t=w_{\widehat\theta_t}.
\label{eq:online_lsif}
\end{equation}
Here, $\lambda\ge0$ controls regularization of the ratio estimator through a fixed penalty $\Omega$. The population counterpart of the first two terms equals
$\frac12\mathbb E_{\mathbf P_0^X}[(w_\theta(X)-w_t(X))^2]$
up to a constant independent of $\theta$. Equation~\ref{eq:online_lsif} therefore approximates the current ratio using source candidates and unlabeled target candidates. Repeating this update with $\mathcal B_t$ at each iteration gives an online density-ratio estimation procedure for the sequence $\{\mathbf P_t^X\}_{t=1}^{T}$.

We incorporate the estimated ratio into CCI by extending the calibration function space from Section~\ref{sec:exchangeable_cci}:
\[
\mathcal G_t
=
\mathcal G+\operatorname{span}\{\widehat w_t\}.
\]
The ratio enters as an additional basis function, while the augmented quantile regression retains its unweighted pinball loss and unpenalized calibration objective. For each fresh candidate $x$, we solve~\eqref{eq:augmented_quantile_calibration} over $\mathcal G_t$ and construct $\mathcal I_t(x)$ through~\eqref{eq:cci_interval}. Finite-dimensional CCI provides coverage under covariate shifts represented by nonnegative functions in its calibration space~\citep{gibbs2024conformalpredictionconditionalguarantees}. Adding $\widehat w_t$ therefore connects the calibration adjustment to the current generation distribution.

For the coverage analysis, suppose that
$\mathbb E_{\mathbf P_0^X}[\widehat w_t(X)]>0$
and define
\begin{equation}
\epsilon_t
=
\frac12
\mathbb E_{\mathbf P_0^X}
\left[
\left|
w_t(X)
-
\frac{\widehat w_t(X)}
{\mathbb E_{\mathbf P_0^X}[\widehat w_t(X)]}
\right|
\right].
\label{eq:weight_approximation_error}
\end{equation}
The normalized ratio defines a candidate distribution relative to $\mathbf P_0^X$, and $\epsilon_t$ is its total variation distance from $\mathbf P_t^X$. This term captures density-ratio estimation error and misspecification of the representation or ratio model. Normalization preserves membership in $\mathcal G_t$, since multiplying a basis function by a positive constant leaves its linear span unchanged.

Algorithm~\ref{alg:proxy_self_evolution_cci} summarizes the resulting search process. The current probe batch determines $\widehat w_t$ and $\mathcal G_t$, while each fresh candidate receives its own test-point calibration. The archive stores the resulting interval rewards, and the cumulative output contains candidates whose entire intervals satisfy the task-specific constraints.

\begin{algorithm}[t]
\caption{\textsc{CISE}}
\label{alg:proxy_self_evolution_cci}
\begin{algorithmic}[1]

\REQUIRE Initial archive $\mathcal H_0$; evolving agent; proxy reward
$\hat r$; data splits $\mathcal D_{\mathrm{fit}}$,
$\mathcal D_{\mathrm{cal}}$, and $\mathcal D_{\mathrm{src}}$;
base calibration space $\mathcal G$; density-ratio model class
$\mathcal W$; parent-selection policies $\{\pi_t\}_{t=0}^{T}$;
task-specific reward constraints; miscoverage level $\alpha$;
probe budget; candidate budget $M$; parent budget $K$;
number of iterations $T$

\ENSURE Interval-qualified candidate set $\mathcal O_T$

\STATE Fit the base quantile model on $\mathcal D_{\mathrm{fit}}$
and fix the score $S$

\STATE Compute $S(X_i,r(X_i))$ for every
$(X_i,r(X_i))\in\mathcal D_{\mathrm{cal}}$

\STATE Initialize the stored rewards in $\mathcal H_0$ using
the intervals from Section~\ref{sec:exchangeable_cci},
and set $\mathcal O_0\leftarrow\emptyset$

\STATE Select $K$ initial parents according to
$\pi_0(\cdot\mid\mathcal H_0)$
and denote them by $\mathcal S_0$

\FOR{$t=1,\ldots,T$}

    \STATE Fix the current generation configuration using
    $\mathcal S_{t-1}$ and their stored rewards

    \STATE Generate an unlabeled probe batch $\mathcal B_t$
    under this configuration without updating the archive

    \STATE Estimate $\widehat w_t$ from
    $\mathcal D_{\mathrm{src}}$ and $\mathcal B_t$
    using~\eqref{eq:online_lsif}

    \STATE Set
    $\mathcal G_t\leftarrow
    \mathcal G+\operatorname{span}\{\widehat w_t\}$

    \STATE Propose fresh candidates
    $\mathcal C_t=\{x_{t,1},\ldots,x_{t,M}\}$
    independently of $\mathcal B_t$
    under the same generation configuration

    \STATE Initialize
    $\mathcal H_t\leftarrow\mathcal H_{t-1}$
    and
    $\mathcal O_t\leftarrow\mathcal O_{t-1}$

    \FORALL{$x\in\mathcal C_t$}

        \STATE Evaluate $\hat r(x)$ and construct
        $\mathcal I_t(x)$ using~\eqref{eq:augmented_quantile_calibration}
        and~\eqref{eq:cci_interval} with $\mathcal G_t$

        \IF{$\mathcal I_t(x)$ is nonempty and bounded}

            \STATE Evaluate the candidate conservatively over
            $\mathcal I_t(x)$ and store it in $\mathcal H_t$

            \IF{every reward value in $\mathcal I_t(x)$
            satisfies the task-specific constraints}
                \STATE Update
                $\mathcal O_t\leftarrow\mathcal O_t\cup\{x\}$
            \ENDIF

        \ENDIF

    \ENDFOR

    \STATE Select $K$ parents according to
    $\pi_t(\cdot\mid\mathcal H_t)$
    and denote them by $\mathcal S_t$

\ENDFOR

\STATE \textbf{return} $\mathcal O_T$

\end{algorithmic}
\end{algorithm}

The following result combines the finite-dimensional CCI guarantee
with the distribution discrepancy in
Equation~\ref{eq:weight_approximation_error}.

\begin{theorem}[Iteration-wise marginal coverage under evolutionary covariate shift]
\label{thm:shift_aware_cci}

For a fixed iteration $t$, condition on the current generation state,
$\mathcal D_{\mathrm{fit}}$, $\mathcal D_{\mathrm{src}}$, and
$\mathcal B_t$. Assume that, conditionally on these quantities, the
calibration pairs remain i.i.d.\ draws from $\mathbf P_0$, and the
fresh pair $(X',r(X'))$ is an independent draw from $\mathbf P_t$
satisfying~\eqref{eq:evolutionary_covariate_shift}.
Suppose that $\widehat w_t$ is nonnegative with positive source
expectation, and that the score $S$ and finite-dimensional space
$\mathcal G_t$ are fixed under this conditioning. If the augmented
calibration over $\mathcal G_t$ is solved exactly with the symmetric
dual-solution rule in Theorem~\ref{thm:exchangeable_cci}, then
\[
\mathbb P
\left(
r(X')\in\mathcal I_t(X')
\right)
\ge
1-\alpha-\epsilon_t,
\]
where $\epsilon_t$ is defined in~\eqref{eq:weight_approximation_error}.
The probability is conditional on the quantities held fixed and
marginal over the calibration data and the fresh candidate--reward pair.

\end{theorem}

The normalized estimated ratio belongs to $\mathcal G_t$, so the
finite-dimensional CCI result gives coverage of at least $1-\alpha$
under its induced candidate distribution.
The calibration distribution remains $\mathbf P_0$ in both settings,
and the conditional reward law is shared. Replacing the induced
candidate distribution by $\mathbf P_t^X$ therefore changes the
coverage probability by at most $\epsilon_t$, establishing
Theorem~\ref{thm:shift_aware_cci}. Averaging over the estimation data
gives the corresponding marginal bound with
$\mathbb E[\epsilon_t]$. The same $1-\alpha$ coverage follows directly whenever $w_t\in\mathcal G_t$.
By the output rule in Section~\ref{sec:problem_setup}, the probability
of returning a false positive is bounded by $\alpha+\epsilon_t$
under the same fixed-iteration setting.

\section{Experiments}

\subsection{Experimental Setup}
\label{sec:experimental_setup}

We evaluate CISE on three multi-constraint material-discovery tasks from the LLEMA benchmark~\citep{abhyankar2026llemaevolutionarysearchllms}: wide-bandgap semiconductors (WBG), which require $E_g \geq 2.5$~eV and $E_{\mathrm{form}} \leq -1.0$~eV/atom; solid-state electrolytes (SSE), which require $E_{\mathrm{form}} \leq -1.0$~eV/atom, $E_g \geq 2.0$~eV, and at least one mobile species from $\{\mathrm{Li},\mathrm{Na},\mathrm{K},\mathrm{Mg},\mathrm{Ca},\mathrm{Al}\}$; and photovoltaic absorbers (PV), which require $0.7 \leq E_g \leq 2.0$~eV, $E_{\mathrm{form}} \leq 0.0$~eV/atom, and compositions restricted to earth-abundant, non-toxic elements. A candidate is counted as a task hit only when all corresponding constraints are satisfied. We provide detailed descriptions of each task in Appendix~\ref{app:detailed_task_descriptions}. We set $\alpha=0.1/ \#\text{constraint}$ for each property. For example, we set $\alpha = 0.05$ for both $E_g$ and $E_{\mathrm{form}}$ in the WBG task. 

We compare CISE against MatterGen~\citep{zeni2025generative}, LLMatDesign~\citep{jia2024llmatdesignautonomousmaterialsdiscovery}, and LLEMA~\citep{abhyankar2026llemaevolutionarysearchllms}. CISE modifies only the archive-admission and output-selection steps of LLEMA; its generator and evolutionary operators remain unchanged. For proxy evaluation, we use MGT in place of CGCNN for band-gap prediction, while otherwise following the LLEMA experimental protocol, including the task definitions, hierarchical proxy-property pipeline, proxy-reward construction, candidate budgets, and method-specific generation settings. For tasks involving multiple predicted properties, we apply CISE separately to each property and return only candidates that satisfy all property-specific CISE criteria. During search, surrogate predictions provide the proxy properties used to score and qualify candidates. We then re-evaluate the candidates returned by each method using DFT implemented in Quantum ESPRESSO~\citep{Giannozzi_2009}. A proxy hit is a candidate satisfying the task constraints under surrogate predictions, whereas a DFT hit satisfies the same physical-property constraints under Quantum ESPRESSO calculations, together with any deterministic compositional restrictions. A proxy hit that fails this final evaluation is recorded as a false positive. We use gpt-5-mini~\citep{singh2026openaigpt5card} as a solver LLM to generate candidates. More detailed experimental settings are provided in Appendix~\ref{app:detailed_experimental_settings}.

\begin{table*}[t]
\centering
\caption{
Comparison of output reliability across three material-discovery tasks.
Each method generates 200 candidate structures.
\textbf{Output} denotes candidates satisfying the method-specific output criterion.
\textbf{TP} and \textbf{FP} denote output candidates that satisfy or violate
the target constraints under DFT evaluation, respectively.
Values in parentheses denote the fraction among output candidates.
}
\label{tab:main_results}

\footnotesize
\renewcommand{\arraystretch}{1.10}
\setlength{\tabcolsep}{3pt}

\begin{tabular*}{\textwidth}{
@{}l@{\extracolsep{\fill}}
ccc
ccc
ccc@{}
}
\toprule
& \multicolumn{3}{c}{\textbf{Wide-bandgap (WBG)}}
& \multicolumn{3}{c}{\textbf{Solid-state electrolytes (SSE)}}
& \multicolumn{3}{c}{\textbf{Photovoltaic absorbers (PV)}} \\
\cmidrule(lr){2-4}
\cmidrule(lr){5-7}
\cmidrule(lr){8-10}

\textbf{Method}
& Output & TP & FP
& Output & TP & FP
& Output & TP & FP \\
\midrule

MatterGen
& 3/200
& 2 (0.67)
& 1 (0.33)
& 2/200
& 1 (0.50)
& 1 (0.50)
& 0/200
& --
& -- \\

LLMatDesign
& 37/200
& 0 (0.00)
& 37 (1.00)
& 2/200
& 0 (0.00)
& 2 (1.00)
& 18/200
& 0 (0.00)
& 18 (1.00) \\

LLEMA
& 102/200
& 94 (0.92)
& 8 (0.08)
& 104/200
& 71 (0.68)
& 33 (0.32)
& 60/200
& 39 (0.65)
& 21 (0.35) \\

\textbf{CISE}
& \textbf{41/200}
& \textbf{41 (1.00)}
& \textbf{0 (0.00)}
& \textbf{51/200}
& \textbf{51 (1.00)}
& \textbf{0 (0.00)}
& \textbf{6/200}
& \textbf{6 (1.00)}
& \textbf{0 (0.00)} \\

\bottomrule
\end{tabular*}

\end{table*}

\begin{figure*}[t]
    \centering
    \includegraphics[width=\textwidth]{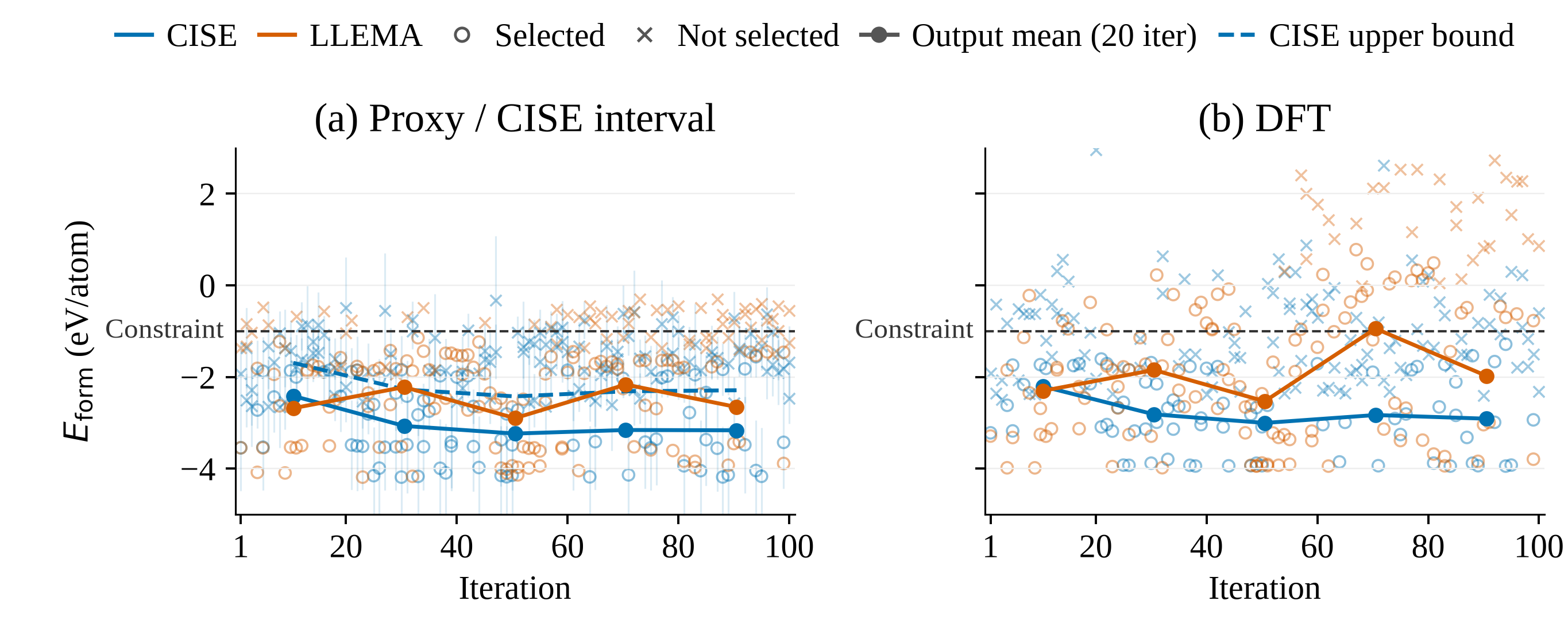}
    \caption{
    Evolution of formation energy on the solid-state electrolyte (SSE) task, where success requires $E_{\mathrm{form}}\leq -1.0$ eV/atom.
    (a) Proxy rewards for LLEMA and reward intervals for CISE; vertical blue bars denote CISE intervals and the dashed blue curve shows their mean upper bound. (b) Corresponding DFT evaluations. Circles and crosses denote selected and non-selected candidates, respectively, and solid curves show the 20-iteration output mean.
    In later iterations, LLEMA produces a pronounced concentration of candidates above the DFT success threshold, visible in the upper-right of (b), including selected false positives. In contrast, CISE shows no comparable accumulation and its output mean remains within the feasible region throughout evolution.
    }
    \label{fig:sse_eform_evolution}
\end{figure*}

\subsection{Empirical results}
\textbf{Main results.} Table~\ref{tab:main_results} shows a trade-off between the number of returned structures and their DFT results. This contrast is particularly clear between LLEMA and CISE on WBG and SSE. LLEMA returns 102 and 104 candidates, respectively, but 7.84\% and 31.73\% fail DFT validation. In contrast, CISE returns 41 WBG candidates and 51 SSE candidates, with no observed false positives among candidates with resolved DFT evaluations: all 41 WBG candidates and 51 SSE candidates are confirmed true positives. A smaller but more reliable shortlist can be valuable when downstream experimental validation is the bottleneck. The objective is not simply to maximize the number of structures satisfying proxy constraints, but to prioritize candidates worth further investigation. Under a limited validation budget, higher output precision can reduce the fraction of resources allocated to candidates that do not satisfy the intended properties. We therefore view the WBG and SSE results as a reliability-yield trade-off. This trade-off is consistent with CISE's joint interval-based selection rule. Whereas proxy-based selection checks whether each point prediction satisfies its constraint, CISE requires the entire prediction interval of every required property to lie within the corresponding feasible region. A candidate can therefore satisfy all point constraints yet be excluded because at least one interval extends beyond an admissible boundary.

\textbf{Calibrated interval rewards facilitate reliable LLM evolution.}
Figure~\ref{fig:sse_eform_evolution} shows a clear difference in the search trajectories of LLEMA and CISE on the SSE task. Although LLEMA continues to assign favorable proxy rewards, its DFT-evaluated candidates increasingly populate the failure region in later iterations.
This is particularly visible in the upper-right of Figure~\ref{fig:sse_eform_evolution}(b), where many late-iteration LLEMA candidates lie above the $E_{\mathrm{form}}=-1.0$ threshold, including selected false positives. In contrast, CISE shows no comparable accumulation, and its output mean remains within the feasible region throughout evolution. This pattern suggests that calibrated interval rewards help prevent overly optimistic proxy evaluations from being repeatedly reinforced during LLM evolution.

\textbf{Comparison of fixed-budget shortlists of LLEMA and CISE.} To examine whether the higher reliability of CISE is simply a consequence of returning fewer candidates, we perform an equal-budget comparison on the PV task. CISE returns six candidates in this experiment, so we restrict LLEMA to the same downstream validation budget of six structures. Among the proxy-qualified LLEMA
outputs, we select the six candidates whose proxy-predicted band gaps are closest to $1.2$ eV, representing a natural point-estimate-based prioritization strategy for PV screening. Importantly, this selection is performed before DFT evaluation, so the true-positive and false-positive labels are unavailable when the candidates are ranked. As shown in Table~\ref{tab:pv_equal_budget}, only three of the six selected LLEMA candidates satisfy the PV band-gap constraint after DFT evaluation, while the remaining three are false positives. In contrast, all six candidates returned by CISE are confirmed as true positives. This result indicates that the reliability gain of CISE cannot be explained solely by its smaller output set: even under the
same validation budget and a stricter post-hoc truncation of the LLEMA output, point-estimate-based prioritization can still select candidates whose apparent quality is caused by proxy error, whereas interval-based selection excludes such uncertain candidates before expensive validation.

\begin{table}[t]
\centering
\caption{
Equal post-evaluation budget comparison on the photovoltaic absorber (PV) task. To match the six candidates returned by CISE, we select six candidates from the LLEMA output whose proxy-predicted band gaps are closest to the target value of $1.2$ eV. Importantly, this ranking is performed before DFT evaluation and therefore does not use the TP/FP labels. Under the same validation budget, three of the six LLEMA candidates are false positives under DFT, whereas all six candidates returned by CISE are true positives.
}
\label{tab:pv_equal_budget}
\resizebox{\linewidth}{!}{
\begin{tabular}{c l c c c | l c c c}
\toprule
& \multicolumn{4}{c|}{LLEMA: top-6 closest to $1.2$ eV}
& \multicolumn{4}{c}{CISE: six returned candidates} \\
\cmidrule(lr){2-5}
\cmidrule(lr){6-9}
Rank
& Material
& Proxy $E_g$
& DFT $E_g$
& Result
& Material
& CISE interval
& DFT $E_g$
& Result \\
\midrule
1
& SrO
& 1.24
& 2.99
& FP
& RbLiO
& $[0.73,\,1.98]$
& 1.24
& TP \\

2
& ZnF$_2$
& 1.25
& 2.32
& FP
& RbNa$_2$OF
& $[1.34,\,1.51]$
& 1.32
& TP \\

3
& MgTiO$_3$
& 1.37
& 1.65
& TP
& KNa$_2$OF
& $[1.19,\,1.62]$
& 1.50
& TP \\

4
& MgO
& 1.54
& 1.39
& TP
& Na$_3$Li$_3$P$_2$
& $[0.92,\,1.72]$
& 1.08
& TP \\

5
& MgZrO$_3$
& 1.57
& 1.07
& TP
& NaLi(MgP)$_2$
& $[0.83,\,1.86]$
& 1.66
& TP \\

6
& CaO
& 1.75
& 3.06
& FP
& KLi$_5$P$_2$
& $[0.89,\,1.73]$
& 1.00
& TP \\
\midrule
\multicolumn{2}{r}{DFT outcome}
& \multicolumn{3}{c|}{3 TP / 3 FP}
& \multicolumn{4}{c}{6 TP / 0 FP} \\
\bottomrule
\end{tabular}
}
\end{table}

\section{Conclusion}
In this work, we study reliable self-evolution when accurate true rewards are too expensive to obtain throughout iterative search. We propose CISE, which replaces point-valued proxy rewards with calibrated reward intervals constructed using conditional conformal inference and iteration-wise online density-ratio estimation. These intervals are used for both evolutionary feedback and final candidate selection, allowing proxy-reward uncertainty to directly influence the search process. Under explicit independence and covariate-shift assumptions, we provide fixed-iteration coverage guarantees and characterize the effect of density-ratio estimation error. Across three material-discovery tasks, CISE consistently returns a smaller but more reliable set of candidates, with no observed false positives among its DFT-validated outputs and improved reliability under a matched validation budget. Our results suggest that calibrated interval rewards provide a practical route toward more reliable self-evolving agents when true rewards are costly to obtain.

\newpage
\bibliography{CISE}
\bibliographystyle{CISE}

\newpage
\appendix
\onecolumn

\begin{algorithm}[t]
\caption{\textsc{Proxy-Reward-Driven Self-Evolving Search}}
\label{alg:proxy_self_evolution}
\begin{algorithmic}[1]

\REQUIRE Initial archive $\mathcal H_0$; evolving agent; proxy reward
$\hat r$; parent-selection policies $\{\pi_t\}_{t=0}^{T}$;
proxy-reward constraints; candidate budget $M$; parent budget $K$;
number of iterations $T$

\ENSURE Proxy-qualified candidate set $\mathcal O_T$

\STATE Initialize $\mathcal O_0\leftarrow\emptyset$

\STATE Select $K$ initial parents according to
$\pi_0(\cdot\mid\mathcal H_0)$ and denote them by $\mathcal S_0$

\FOR{$t=1,\ldots,T$}

    \STATE Propose
    $\mathcal C_t=\{x_{t,1},\ldots,x_{t,M}\}$
    using $\mathcal S_{t-1}$ and their stored proxy rewards

    \STATE Evaluate $\hat r(x)$ for every $x\in\mathcal C_t$

    \STATE Update
    $\mathcal H_t\leftarrow
    \mathcal H_{t-1}\cup
    \{(x,\hat r(x)):x\in\mathcal C_t\}$

    \STATE Update
    $\mathcal O_t\leftarrow
    \mathcal O_{t-1}\cup
    (\mathcal C_t\cap\mathcal F_{\hat r})$

    \STATE Select $K$ parents according to
    $\pi_t(\cdot\mid\mathcal H_t)$
    and denote them by $\mathcal S_t$

\ENDFOR

\STATE \textbf{return} $\mathcal O_T$

\end{algorithmic}
\end{algorithm}

\begin{table*}[t]
\centering
\caption{Notation used for proxy-reward-driven material evolution
and conditional conformal inference.}
\label{tab:evolution_notation}
\small
\setlength{\tabcolsep}{6pt}
\renewcommand{\arraystretch}{1.12}
\begin{tabularx}{\textwidth}{@{}p{0.23\textwidth}X@{}}
\toprule
Notation & Description \\
\midrule

\multicolumn{2}{@{}l}{\textit{Material search and rewards}}\\
\addlinespace[2pt]

$\mathcal X$; $X,x$; $X'$
&
Space of material candidates; a random candidate and its
realization; a fresh candidate at the current iteration.
\\

$\hat r(x)$; $r(x)$
&
Computationally inexpensive proxy reward and expensive true
reward of candidate $x$.
\\

$\mathcal H_t$
&
Archive after iteration $t$, containing candidates and their
stored rewards. Algorithm~\ref{alg:proxy_self_evolution} stores
proxy rewards, while Algorithm~\ref{alg:proxy_self_evolution_cci}
stores rewards evaluated conservatively over prediction intervals.
\\

$\pi_t(\cdot\mid\mathcal H_t)$; $\mathcal S_t$
&
Parent-selection policy applied to $\mathcal H_t$ and the
resulting set of $K$ parents used to guide the following iteration.
\\

$\mathcal C_t$; $\mathcal O_t$
&
Candidate batch proposed at iteration $t$ and cumulative output
set through that iteration.
\\

$\mathcal F_{\hat r}$; $\mathcal F_r$
&
Sets of candidates satisfying the task-specific constraints
according to the proxy reward and the true reward, respectively.
\\

$M$; $K$; $T$
&
Number of fresh candidates evaluated per iteration, number of
selected parents, and number of iterations.
\\

\addlinespace[3pt]
\multicolumn{2}{@{}l}{\textit{Conditional conformal inference}}\\
\addlinespace[2pt]

$\mathcal D$; $n$; $n_{\mathrm{cal}}$
&
Labeled source dataset
$\mathcal D=\{(X_i,r(X_i))\}_{i=1}^{n}$,
its size, and the number of calibration pairs.
\\

$\mathcal D_{\mathrm{fit}}$;
$\mathcal D_{\mathrm{cal}}$;
$\mathcal D_{\mathrm{src}}$
&
Mutually disjoint subsets used to fit the base quantile model,
calibrate prediction intervals, and provide the source reference
for density-ratio estimation.
\\

$\mathcal I_t(x)$
&
Prediction interval for $r(x)$ constructed at iteration $t$.
A candidate is returned only when its interval is nonempty,
bounded, and satisfies the task-specific constraints throughout.
\\

$\alpha$
&
Nominal marginal miscoverage level for a fresh candidate,
evaluated separately at each iteration.
\\

$S(x,u)$; $S_i$
&
Fixed nonconformity score for a hypothetical reward $u$ and
calibration score $S_i=S(X_i,r(X_i))$.
\\

$\mathcal G$; $\mathcal G_t$
&
Base finite-dimensional calibration space containing the
constant function and its iteration-specific extension
$\mathcal G_t=\mathcal G+\operatorname{span}\{\widehat w_t\}$.
\\

$s$; $\widehat q_s$
&
Imputed test score and fitted calibration function obtained
from augmented quantile regression with that score.
\\

$\ell_\alpha(a,b)$
&
Pinball loss for the $(1-\alpha)$-quantile,
$\ell_\alpha(a,b)=(1-\alpha)(b-a)_++\alpha(a-b)_+$.
\\

$\eta_s(X')$
&
Dual variable associated with the fresh candidate in the
augmented calibration problem, taking values in
$[-\alpha,1-\alpha]$.
\\

$\widehat\tau_t(x)$
&
Candidate-specific score cutoff obtained from the augmented
dual boundary and inverted to construct $\mathcal I_t(x)$.
\\

\addlinespace[3pt]
\multicolumn{2}{@{}l}{\textit{Evolutionary covariate-shift correction}}\\
\addlinespace[2pt]

$\mathbf P_0$; $\mathbf P_t$
&
Joint distributions of a candidate and its true reward in the
source data and under the fixed generation configuration at
iteration $t$, respectively.
\\

$\mathbf P_0^X$; $\mathbf P_t^X$
&
Marginal distributions over material candidates under
$\mathbf P_0$ and $\mathbf P_t$.
\\

$\phi(x)$
&
Fixed material representation used in the calibration basis
and density-ratio estimation.
\\

$\mathcal B_t$
&
Unlabeled probe batch generated under the current fixed
configuration and used only for density-ratio estimation.
\\

$w_t(x)$; $\widehat w_t(x)$
&
Oracle density ratio
$w_t(x)=\frac{d\mathbf P_t^X}{d\mathbf P_0^X}(x)$
and its online estimate from
$\mathcal D_{\mathrm{src}}$ and $\mathcal B_t$.
\\

$\mathcal W$; $W$
&
Density-ratio model class with values in $[0,W]$ and the
bounded-overlap constant satisfying $w_t(x)\le W$.
\\

$\epsilon_t$
&
Total variation distance between $\mathbf P_t^X$ and the
candidate distribution induced by the normalized estimated
ratio, as defined in~\eqref{eq:weight_approximation_error}.
\\

\bottomrule
\end{tabularx}
\end{table*}

\section{Detailed experimental settings}
\label{app:detailed_experimental_settings}

\subsection{Implementations}
\textbf{Calibration set.}
We construct the calibration set $\mathcal{D}_{\mathrm{cal}}$ using 10,000 material structures from JARVIS-QE-TB (QETB)~\citep{Choudhary_2025}. The calibration samples are separated from the model-fitting data at the composition-group level to avoid overlap between the two sets.

\textbf{Model-fitting set.}
We use $\mathcal{D}_{\mathrm{fit}}$ consisting of 423,205 structures from QETB to train the conditional quantile regression model. The proxy models remain fixed throughout the evolution process.

\textbf{Probe generation.}
At each generation context, we generate eight valid probe structures using the same parent structures and prompt context as the main candidate generation. These probes are used to characterize the current candidate distribution for online density-ratio estimation.

\textbf{Conditional quantile regression.}
We implement the conditional quantile regression model using a gradient-boosted regression model (HistGradientBoostingRegressor) with quantile loss. The model is trained to predict the conditional $0.95$ quantile of absolute proxy error. Its input consists of 27 structural features together with the corresponding proxy prediction.

\textbf{CCI features.}
For conditional conformal inference, the structural features are used to capture heterogeneous proxy uncertainty across materials. The online density ratio estimated from the probe structures and a fixed source set is additionally incorporated into the finite-dimensional Gibbs calibration basis, allowing the calibration procedure to adapt to the candidate distribution at the current generation step.

\textbf{DFT evaluation.}
Selected material structures are evaluated using Quantum ESPRESSO under a PBEsol fixed-structure protocol consistent with the QETB calculations. We use the Modified GBRV PBEsol pseudopotentials and evaluate the target material properties independently from the proxy models.

\subsection{Detailed task descriptions}
\label{app:detailed_task_descriptions}
\begin{itemize}
    \item \textbf{Wide-bandgap semiconductors (WBG):}
    the target constraints are
    $E_g\geq2.5$~eV and
    $E_{\mathrm{form}}\leq-1.0$~eV/atom.
    Here,
    $\hat r(X)=(\widehat E_g,\widehat E_{\mathrm{form}})$
    and
    $r(X)=(E_g,E_{\mathrm{form}})$,
    where the latter quantities are obtained by DFT.
    Since a larger predicted $E_g$ and a smaller predicted
    $E_{\mathrm{form}}$ can make a candidate appear more favorable,
    \[
    Y(X)
    =
    \mathbb I
    \left\{
    \widehat E_g-E_g\leq\delta_{E_g},
    \;
    E_{\mathrm{form}}-\widehat E_{\mathrm{form}}
    \leq\delta_{E_{\mathrm{form}}}
    \right\}.
    \]

    \item \textbf{Solid-state electrolytes (SSE):}
    the constraints are
    $E_{\mathrm{form}}\leq-1.0$~eV/atom,
    $E_g\geq2.0$~eV, and the presence of at least one mobile species from
    $\{\mathrm{Li},\mathrm{Na},\mathrm{K},\mathrm{Mg},\mathrm{Ca},\mathrm{Al}\}$.
    We use
    $\hat r(X)=(\widehat E_g,\widehat E_{\mathrm{form}})$
    and
    $r(X)=(E_g,E_{\mathrm{form}})$,
    with
    \[
    Y(X)
    =
    \mathbb I
    \left\{
    \widehat E_g-E_g\leq\delta_{E_g},
    \;
    E_{\mathrm{form}}-\widehat E_{\mathrm{form}}
    \leq\delta_{E_{\mathrm{form}}}
    \right\}.
    \]
    The mobile-species requirement is determined directly from the
    composition and therefore does not require proxy calibration.

    \item \textbf{Photovoltaic absorbers (PV):}
    the constraints are
    $0.7\leq E_g\leq2.0$~eV,
    $E_{\mathrm{form}}\leq0.0$~eV/atom, and compositions restricted to
    earth-abundant, non-toxic elements.
    We again use
    $\hat r(X)=(\widehat E_g,\widehat E_{\mathrm{form}})$
    and
    $r(X)=(E_g,E_{\mathrm{form}})$.
    Because $E_g$ is constrained on both sides, its proxy error is controlled
    in both directions:
    \[
    Y(X)
    =
    \mathbb I
    \left\{
    |\widehat E_g-E_g|\leq\delta_{E_g},
    \;
    E_{\mathrm{form}}-\widehat E_{\mathrm{form}}
    \leq\delta_{E_{\mathrm{form}}}
    \right\}.
    \]
    The elemental restriction is evaluated directly from the composition.

\end{itemize}

\begin{figure}[!t]
    \centering
    \includegraphics[width=\linewidth]{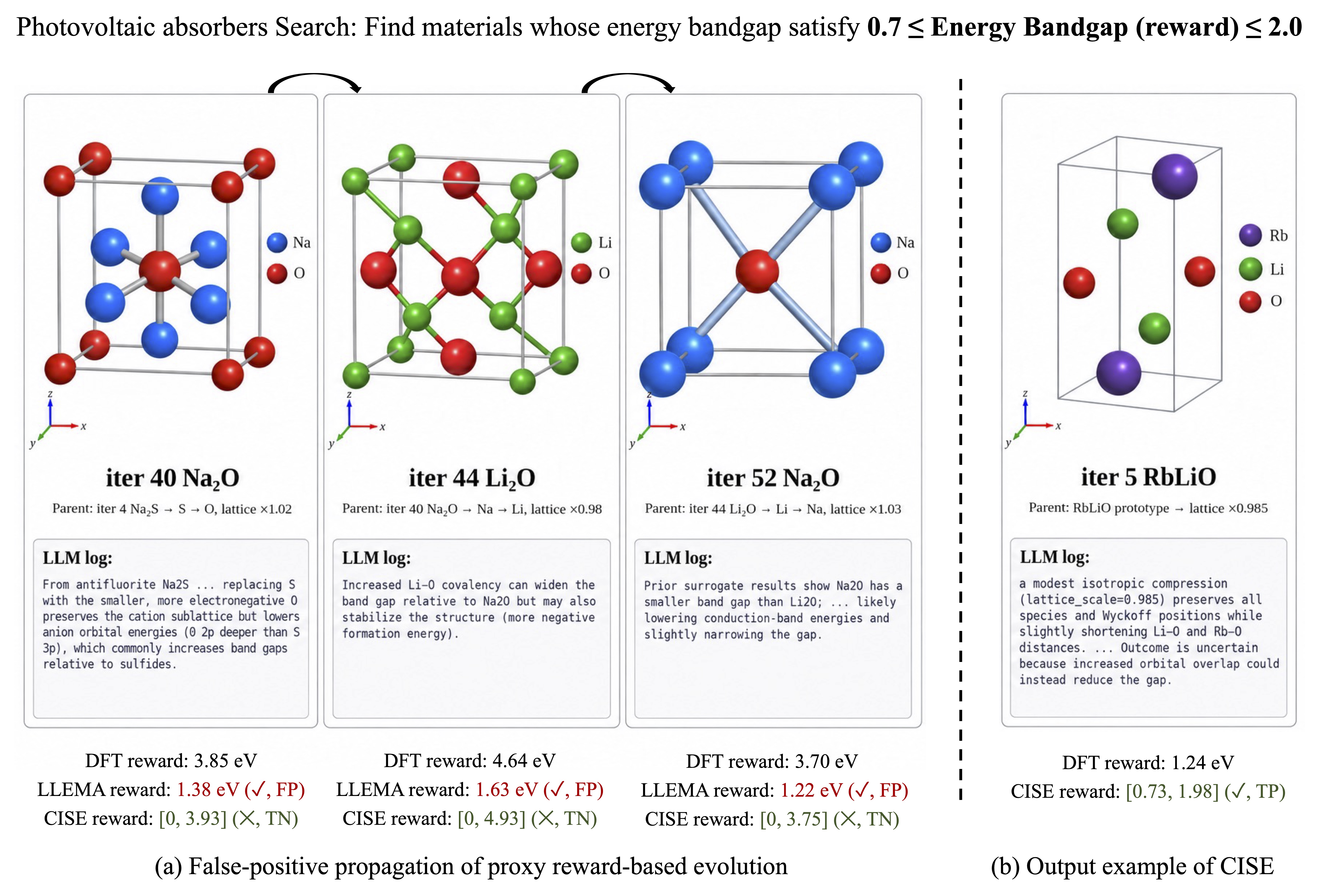}
    \caption{
    Qualitative comparison of proxy-driven evolution and CISE on the photovoltaic absorber task, with target band gap $0.7 \leq E_g \leq 2.0$~eV.
    (a) LLEMA reuses proxy-qualified candidates along the parent--child lineage $\mathrm{Na_2O} \rightarrow \mathrm{Li_2O} \rightarrow
    \mathrm{Na_2O}$. All three candidates satisfy the point-estimate constraint but fail DFT validation (FP). The recorded LLM design justifications illustrate how previous surrogate evaluations
    inform subsequent proposals, allowing false positives to become reusable search feedback. The displayed CISE intervals extend beyond the target range, so these candidates would be excluded
    by the interval-based output rule. (b) A $\mathrm{RbLiO}$ candidate from a separate CISE run has its entire band-gap interval within the target range and also satisfies the constraint under DFT evaluation. Arrows indicate parent--child relations; checkmarks and crosses indicate acceptance and rejection, respectively.
    }
    \label{fig:pv_qualitative}
\end{figure}

\section{Additional Empirical Results}

\textbf{Interval rewards mitigate false-positive propagation.} Figure~\ref{fig:pv_qualitative} qualitatively illustrates how proxy errors can propagate through self-evolving search and how CISE changes this feedback process. In the LLEMA trajectory in Figure~\ref{fig:pv_qualitative}(a), the agent repeatedly reuses proxy-qualified candidates along the parent-child lineage $\mathrm{Na_2O} \rightarrow \mathrm{Li_2O} \rightarrow \mathrm{Na_2O}$. Although all three candidates satisfy the photovoltaic band-gap constraint according to their proxy predictions, DFT evaluation shows that each is a false positive. Importantly, the recorded LLM rationales explicitly refer to previous surrogate evaluations when proposing subsequent candidates, showing that an erroneous proxy reward can become reusable feedback rather than remaining an isolated output error. In contrast, the corresponding CISE intervals extend beyond the feasible band-gap range, preventing these candidates from being treated as reliable successes under the interval-based criterion. Figure~\ref{fig:pv_qualitative}(b) shows the complementary case: the full CISE interval for $\mathrm{RbLiO}$ lies within the target range, and the candidate is subsequently confirmed as a true positive by DFT. Together, these examples illustrate how interval rewards can suppress the propagation of uncertain proxy successes while retaining candidates whose uncertainty-aware rewards remain feasible.

\section{Proofs}
\label{app:cci_proofs}

Throughout this section, we fix an iteration $t$ and write
$N=n_{\mathrm{cal}}+1$.
All scores, basis functions, and solution rules are taken to be
measurable.
Symmetry of the dual-solution rule means that permuting the augmented
observations permutes the selected dual coordinates in the same way.
We interpret the supremum in
Equation~\ref{eq:cci_interval} in the extended real line, with
$\sup\emptyset=-\infty$.

\subsection{Proof of Theorem~\ref{thm:exchangeable_cci}}
\label{app:proof_exchangeable_cci}

\begin{proof}
Condition on the independently fixed score $S$ and calibration
space $\mathcal G$.
Set $X_N=X'$ and write
$S_i=S(X_i,r(X_i))$ for $i=1,\ldots,N$.
At the true test score $S_N$, the augmented calibration problem in~\eqref{eq:augmented_quantile_calibration} becomes
\[
    \min_{q\in\mathcal G}
    \sum_{i=1}^{N}\ell_\alpha(q(X_i),S_i).
\]
The pinball loss admits the representation
\[
    \ell_\alpha(a,b)
    =
    \max_{-\alpha\leq\eta\leq1-\alpha}\eta(b-a).
\]
Since $\mathcal G$ is a finite-dimensional linear space,
the augmented problem is a finite-dimensional linear program.
It is feasible and has a finite optimal value: its objective is
nonnegative, and the choice $q=0$ gives a finite value.
Linear-programming duality therefore yields the dual problem
\begin{equation}
\begin{aligned}
    \max_{\eta\in\mathbb R^N}\quad
        &\sum_{i=1}^{N}\eta_i S_i \\
    \text{subject to}\quad
        &-\alpha\leq\eta_i\leq1-\alpha,
          \qquad i=1,\ldots,N,\\
        &\sum_{i=1}^{N}\eta_i g(X_i)=0,
          \qquad \text{for every }g\in\mathcal G.
\end{aligned}
\label{eq:app_augmented_dual}
\end{equation}
The last constraint follows because the coefficients of
$q\in\mathcal G$ are unrestricted.
The dual feasible set is nonempty and compact, so an optimal
dual solution exists.

Let $\eta=(\eta_1,\ldots,\eta_N)$ be the optimal dual solution
selected by the symmetric rule.
Since the constant function belongs to $\mathcal G$,
\eqref{eq:app_augmented_dual} implies
\[
    \sum_{i=1}^{N}\eta_i=0.
\]
The box constraints also give, for every $i$,
\[
    \mathbf 1\{\eta_i=1-\alpha\}
    \leq \eta_i+\alpha.
\]
Consequently,
\[
    \sum_{i=1}^{N}\mathbf 1\{\eta_i=1-\alpha\}
    \leq
    \sum_{i=1}^{N}(\eta_i+\alpha)
    =N\alpha.
\]

Under the assumptions of the theorem, the $N$ augmented pairs
are independent draws from $\mathbf P_0$.
The fixed score and symmetric dual-solution rule therefore make
the selected dual coordinates exchangeable.
It follows that
\[
\begin{aligned}
    \mathbb P(\eta_N=1-\alpha)
    &=
    \frac{1}{N}
    \mathbb E\left[
        \sum_{i=1}^{N}\mathbf 1\{\eta_i=1-\alpha\}
    \right]\\
    &\leq\alpha.
\end{aligned}
\]

At the true test score, $\eta_N=\eta_{S_N}(X_N)$.
If $\eta_N<1-\alpha$, then $S_N$ belongs to the set whose
supremum defines $\widehat\tau_t(X_N)$.
Thus,
$S_N\leq\widehat\tau_t(X_N)$, which implies
$r(X_N)\in\mathcal I_t(X_N)$ by~\eqref{eq:cci_interval}.
We therefore have the pointwise inequalities
\begin{equation}
    \mathbf 1\{r(X_N)\notin\mathcal I_t(X_N)\}
    \leq
    \mathbf 1\{\eta_N=1-\alpha\}
    \leq
    \eta_N+\alpha.
    \label{eq:app_dual_miscoverage}
\end{equation}
Combining the first inequality with the boundary-probability
bound gives
\[
    \mathbb P\bigl(r(X')\notin\mathcal I_t(X')\bigr)
    \leq\alpha.
\]
Averaging over any independently fitted quantities conditioned
on above proves the stated guarantee.

Only an inclusion between the miscoverage event and the dual
boundary event is needed.
In particular, the argument remains valid when the true score
equals the cutoff or when optimal solutions are not unique,
provided that the selected dual solution follows the stated
symmetric rule.
\end{proof}

We also verify the monotonicity of $\eta_s(X')$ used in
Section~\ref{sec:exchangeable_cci}.
Fix the calibration observations, the fresh candidate $X'$,
and the calibration space.
For $s_1<s_2$, let $\eta^{(1)}$ and $\eta^{(2)}$ be optimal dual
solutions with imputed test scores $s_1$ and $s_2$, respectively.
The dual feasible set in~\eqref{eq:app_augmented_dual} does not depend on the
imputed score.
Optimality therefore gives
\[
\begin{aligned}
    \sum_{i=1}^{N-1}S_i\eta_i^{(1)}
        +s_1\eta_N^{(1)}
    &\geq
    \sum_{i=1}^{N-1}S_i\eta_i^{(2)}
        +s_1\eta_N^{(2)},\\
    \sum_{i=1}^{N-1}S_i\eta_i^{(2)}
        +s_2\eta_N^{(2)}
    &\geq
    \sum_{i=1}^{N-1}S_i\eta_i^{(1)}
        +s_2\eta_N^{(1)}.
\end{aligned}
\]
Adding these inequalities yields
\[
    (s_2-s_1)
    \bigl(\eta_N^{(2)}-\eta_N^{(1)}\bigr)
    \geq0.
\]
Hence $\eta_{s_2}(X')\geq\eta_{s_1}(X')$.
This argument applies to any choices of optimal dual solutions
and remains valid when $\mathcal G$ is replaced by a fixed
$\mathcal G_t$.

\subsection{Proof of Theorem~\ref{thm:shift_aware_cci}}
\label{app:proof_shift_aware_cci}

\begin{proof}
Condition on the current generation state,
$\mathcal D_{\mathrm{fit}}$, $\mathcal D_{\mathrm{src}}$,
and $\mathcal B_t$, as specified in the theorem.
All probabilities and expectations below are conditional on
these quantities unless stated otherwise.
In particular, the score $S$, estimated ratio $\widehat w_t$,
and calibration space $\mathcal G_t$ are fixed.
By assumption, the calibration pairs retain their independent
$\mathbf P_0$ laws under this conditioning.

Define
\[
    \mu_t
    =
    \mathbb E_{\mathbf P_0^X}[\widehat w_t(X)],
    \qquad
    \widetilde w_t(x)
    =
    \frac{\widehat w_t(x)}{\mu_t}.
\]
The bounded ratio model and the positive-source-expectation
assumption give $0<\mu_t<\infty$.
Therefore, $\widetilde w_t$ is nonnegative and satisfies
\[
    \mathbb E_{\mathbf P_0^X}[\widetilde w_t(X)]=1,
    \qquad
    \widetilde w_t\in\mathcal G_t.
\]
The latter inclusion follows from
$\widehat w_t\in\mathcal G_t$ and linearity of $\mathcal G_t$.

Let $\widetilde{\mathbf P}_t$ be the joint distribution defined by
\[
    \frac{d\widetilde{\mathbf P}_t}{d\mathbf P_0}(x,u)
    =
    \widetilde w_t(x).
\]
This is a probability distribution because
$\widetilde w_t$ is nonnegative and has source expectation one.
It changes only the candidate marginal while preserving the
conditional reward law of $\mathbf P_0$.

First consider an auxiliary experiment in which all $N$ augmented
pairs are independent draws from $\mathbf P_0$.
Let $\mathbb E_0$ denote expectation in this experiment, with the
conditioned quantities above still held fixed.
For each realization of the augmented sample, solve the same
calibration problem over $\mathcal G_t$ at the true test score.
Let $\eta$ denote the selected optimal dual solution.

The dual formulation
in~\eqref{eq:app_augmented_dual} applies with
$\mathcal G_t$ in place of $\mathcal G$.
Since $\widetilde w_t\in\mathcal G_t$, dual feasibility gives
\[
    \sum_{i=1}^{N}\widetilde w_t(X_i)\eta_i=0.
\]
Under the auxiliary source experiment, the augmented pairs are
exchangeable.
Since $\widetilde w_t$ is fixed and the dual-solution rule is
symmetric, the pairs
$(\widetilde w_t(X_i),\eta_i)$ are also exchangeable.
Consequently,
\[
\begin{aligned}
    \mathbb E_0[\widetilde w_t(X_N)\eta_N]
    &=
    \frac{1}{N}
    \mathbb E_0\left[
        \sum_{i=1}^{N}\widetilde w_t(X_i)\eta_i
    \right]\\
    &=0.
\end{aligned}
\]
These expectations are well-defined because the dual variables
are bounded and $\mathbb E_0[\widetilde w_t(X_N)]=1$.

The pointwise bound
in~\eqref{eq:app_dual_miscoverage} also applies to the
calibration problem over $\mathcal G_t$.
Multiplying it by $\widetilde w_t(X_N)\geq0$ and taking
expectations gives
\begin{equation}
\begin{aligned}
    &\mathbb E_0\left[
        \widetilde w_t(X_N)
        \mathbf 1\{r(X_N)\notin\mathcal I_t(X_N)\}
    \right]\\
    &\qquad\leq
    \mathbb E_0\left[
        \widetilde w_t(X_N)(\eta_N+\alpha)
    \right]
    =\alpha.
\end{aligned}
\label{eq:app_weighted_miscoverage}
\end{equation}

Changing measure for the last observation only, the left-hand
side of~\eqref{eq:app_weighted_miscoverage} is exactly
the miscoverage probability when the calibration pairs are
independent draws from $\mathbf P_0$ and the independent fresh
pair has law $\widetilde{\mathbf P}_t$.
Thus, the prediction interval has miscoverage at most $\alpha$
under this auxiliary target law.
The fitted score, estimated ratio, calibration space, and
interval-construction rule remain unchanged in this comparison.

It remains to replace $\widetilde{\mathbf P}_t$ by the actual
target distribution $\mathbf P_t$.
The covariate-shift assumption in~\eqref{eq:evolutionary_covariate_shift} implies
\[
    \frac{d\mathbf P_t}{d\mathbf P_0}(x,u)=w_t(x).
\]
Since both target laws preserve the conditional reward law of
$\mathbf P_0$, their total variation distance is
\[
\begin{aligned}
    d_{\mathrm{TV}}
    \bigl(\mathbf P_t,\widetilde{\mathbf P}_t\bigr)
    &=
    \frac12
    \mathbb E_{\mathbf P_0^X}
    \left[
        |w_t(X)-\widetilde w_t(X)|
    \right]\\
    &=\epsilon_t,
\end{aligned}
\]
where the last equality follows from~\eqref{eq:weight_approximation_error}.

The calibration sample has the same independent law under both
target distributions.
Therefore,
\[
\begin{aligned}
    &d_{\mathrm{TV}}\left(
        \mathbf P_0^{\otimes n_{\mathrm{cal}}}\otimes\mathbf P_t,\,
        \mathbf P_0^{\otimes n_{\mathrm{cal}}}
            \otimes\widetilde{\mathbf P}_t
    \right)\\
    &\qquad=
    d_{\mathrm{TV}}
    \bigl(\mathbf P_t,\widetilde{\mathbf P}_t\bigr)
    =\epsilon_t.
\end{aligned}
\]
Applying the defining total-variation bound to the miscoverage
event shows that replacing the auxiliary target law by
$\mathbf P_t$ increases its probability by at most $\epsilon_t$.
Combining this with~\eqref{eq:app_weighted_miscoverage} yields
\[
    \mathbb P\bigl(r(X')\notin\mathcal I_t(X')\bigr)
    \leq\alpha+\epsilon_t,
\]
or equivalently,
\[
    \mathbb P\bigl(r(X')\in\mathcal I_t(X')\bigr)
    \geq1-\alpha-\epsilon_t.
\]
This is the conditional guarantee stated in the theorem.

When the stated conditional assumptions hold almost surely,
averaging over all quantities conditioned on above gives the
unconditional bound
\[
    \mathbb P\bigl(r(X')\in\mathcal I_t(X')\bigr)
    \geq1-\alpha-\mathbb E[\epsilon_t].
\]

Finally, if $w_t\in\mathcal G_t$, the weighted argument can be
applied directly with $w_t$ in place of $\widetilde w_t$.
Indeed, $w_t$ is nonnegative and
$\mathbb E_{\mathbf P_0^X}[w_t(X)]=1$, while dual feasibility
gives $\sum_{i=1}^{N}w_t(X_i)\eta_i=0$.
Changing measure then leads directly to the actual target law
$\mathbf P_t$, so no total-variation comparison is needed.
The resulting coverage is at least $1-\alpha$.
In particular, this applies when $\widehat w_t=w_t$, in which
case $\epsilon_t=0$.
\end{proof}

For completeness, the false-positive implication follows directly
from the output rule.
Let $A_t$ be the acceptance event defined in
Section~\ref{sec:problem_setup}.
On $A_t$, every reward value in $\mathcal I_t(X')$ satisfies
the task-specific constraints.
If $r(X')$ also belongs to this interval, then
$X'\in\mathcal F_r$.
Hence
\[
    A_t\cap\{X'\notin\mathcal F_r\}
    \subseteq
    \{r(X')\notin\mathcal I_t(X')\}.
\]
Under the assumptions of
Theorem~\ref{thm:exchangeable_cci}, this gives
$\mathbb P(A_t,\,X'\notin\mathcal F_r)\leq\alpha$.
Under the conditioning and assumptions of
Theorem~\ref{thm:shift_aware_cci}, it gives
\[
    \mathbb P(A_t,\,X'\notin\mathcal F_r)
    \leq\alpha+\epsilon_t.
\]
These bounds concern the joint event that a fresh proposal is
accepted and is a false positive.
They do not assert the same bounds conditional on acceptance
or simultaneously over the full search trajectory.

\end{document}